\documentclass{article} % For LaTeX2e

\usepackage[utf8]{inputenc}
\usepackage{graphics} 
\usepackage{graphicx} % more modern
\usepackage{times} 
\usepackage{amsmath} 
\usepackage{amsthm}
\usepackage{amssymb}
\usepackage{bm}
\usepackage{newtxmath}
\usepackage{newtxtext}
\usepackage{array}
\usepackage{soul}
\usepackage{xspace}
\usepackage[bottom]{footmisc}
\usepackage{url}
\usepackage{booktabs}       % professional-quality tables
\usepackage{tabu}
\usepackage{array}
\usepackage{layouts}
\usepackage[dvipsnames]{xcolor}

\usepackage{microtype}
\usepackage{graphicx}
\usepackage{subfigure}
\usepackage{booktabs} % for professional tables

\usepackage{amsmath,amsfonts,bm}

\def\eqref#1{equation~\ref{#1}}
\def\Eqref#1{Equation~\ref{#1}}
\def\1{\bm{1}}

\def\va{{\bm{a}}}

\def\vu{{\bm{u}}}

\def\vw{{\bm{w}}}
\def\vx{{\bm{x}}}

\def\vz{{\bm{z}}}

\def\mB{{\bm{B}}}

\def\mL{{\bm{L}}}

\def\mQ{{\bm{Q}}}

\DeclareMathAlphabet{\mathsfit}{\encodingdefault}{\sfdefault}{m}{sl}
\SetMathAlphabet{\mathsfit}{bold}{\encodingdefault}{\sfdefault}{bx}{n}

\DeclareMathOperator*{\argmax}{arg\,max}
\DeclareMathOperator*{\argmin}{arg\,min}

\usepackage[shortlabels]{enumitem}
\setlist[enumerate]{nosep}

\theoremstyle{plain}
\newtheorem{theorem}{Theorem}
\newtheorem*{theorem*}{Theorem}
\newenvironment{hproof}{\proof}{\endproof}

\usepackage{hyperref}

\usepackage[accepted]{icml2021}

\icmltitlerunning{Continuous Fitted Value Iteration}

\begin{document}
\twocolumn[
\icmltitle{Value Iteration in Continuous Actions, States and Time}



\icmlsetsymbol{equal}{*}

\begin{icmlauthorlist}
\icmlauthor{Michael Lutter}{nvidia,tuda}
\icmlauthor{Shie Mannor}{nvidia,technion}
\icmlauthor{Jan Peters}{tuda}
\icmlauthor{Dieter Fox}{nvidia,washington}
\icmlauthor{Animesh Garg}{nvidia,toronto}
\end{icmlauthorlist}

\icmlaffiliation{technion}{Technion, Israel Institute of Technology}
\icmlaffiliation{nvidia}{NVIDIA}
\icmlaffiliation{toronto}{University of Toronto \& Vector Institute}
\icmlaffiliation{washington}{University of Washington}
\icmlaffiliation{tuda}{Technical University of Darmstadt}

\icmlcorrespondingauthor{Michael Lutter}{michael@robot-learning.de}
% \icmlcorrespondingauthor{Eee Pppp}{ep@eden.co.uk}

\icmlkeywords{Machine Learning, ICML}

\vskip 0.3in
]


\printAffiliationsAndNotice{}  % leave blank if no need to mention equal contribution

% Authors must not appear in the submitted version. They should be hidden
% as long as the \iclrfinalcopy macro remains commented out below.
% Non-anonymous submissions will be rejected without review.

\begin{abstract}
% 1) Stating the problem: & % 2) Say why it is interesting:
Classical value iteration approaches are not applicable to environments with continuous states and actions. For such environments, the states and actions are usually discretized, which leads to an exponential increase in computational complexity. 
%
% 3) Say what your solution is / what it achieves:
In this paper, we propose continuous fitted value iteration (cFVI). This algorithm enables dynamic programming for continuous states and actions with a known dynamics model. Leveraging the continuous-time formulation, the optimal policy can be derived for non-linear control-affine dynamics. This closed-form solution enables the efficient extension of value iteration to continuous environments. 
%
% 4) Say what follows from your solution:
We show in non-linear control experiments that the dynamic programming solution obtains the same quantitative performance as deep reinforcement learning methods in simulation but excels when transferred to the physical system. The policy obtained by cFVI is more robust to changes in the dynamics despite using only a deterministic model and without explicitly incorporating robustness in the optimization. Videos of the physical system are available at \url{https://sites.google.com/view/value-iteration}.
\end{abstract}
\vspace{-2.5em}
\section{Introduction} \vspace{-0.5em}
Reinforcement learning (RL) maximizes the scalar rewards~$r$ by trying different actions~$\vu$ selected by the policy~$\pi$. For a deterministic policy and deterministic dynamics, the discrete time optimization is described by
\begin{align}
    \pi^{*}(\vx) = \argmax_{\pi} \sum_{i=0}^{\infty} \gamma^{i} r(\vx_i, \vu_i), \label{eq:discrete_rl}
\end{align}
with the dynamics $\vx_{i+1} = f(\vx_i, \vu_i)$, the state $\vx$ and discounting factor $\gamma \in \left[0, 1\right)$ \cite{bellman1957dynamic, puterman1994markov}. The discrete time framework has been very successful to learn policies for various applications including robotics \cite{silver2017mastering, openai2019learning,da2020learning}. However, the fixed pre-determined time discretization is a limitation when dealing with physical phenomena evolving in continuous time. For such environments, the time-step can be chosen to fit the environment. For example, in robotics the control frequency is only limited by the sampling frequencies of sensors and actuators, which is commonly much higher than the control frequency of deep RL agents. To avoid the fixed time steps, we use the continuous-time RL objective to frame an optimization that is agnostic to the time-discretization. Furthermore, this formulation simplifies the optimization as many robot dynamics are control affine at the continuous-time limit. 

To solve the continuous-time RL problem, we use classical value iteration where the dynamics and reward function are known. For robotics, this setting is commonly a reasonable assumption. The reward is known as it is manually tuned to achieve the task without violating system constraints. The equations of motion and system parameters of robots are approximately known. The model is not perfect as the actuators are not ideal and the physical parameters vary slightly. However, the overall system behavior is approximately correct. The same assumption is also used within the vast sim2real literature to enable the sample efficient transfer to the physical world \cite{ramos2019bayessim, chebotar2019closing,xie2021dynamics}.   

\begin{figure}[t]
    \centering
    \includegraphics[width=\columnwidth]{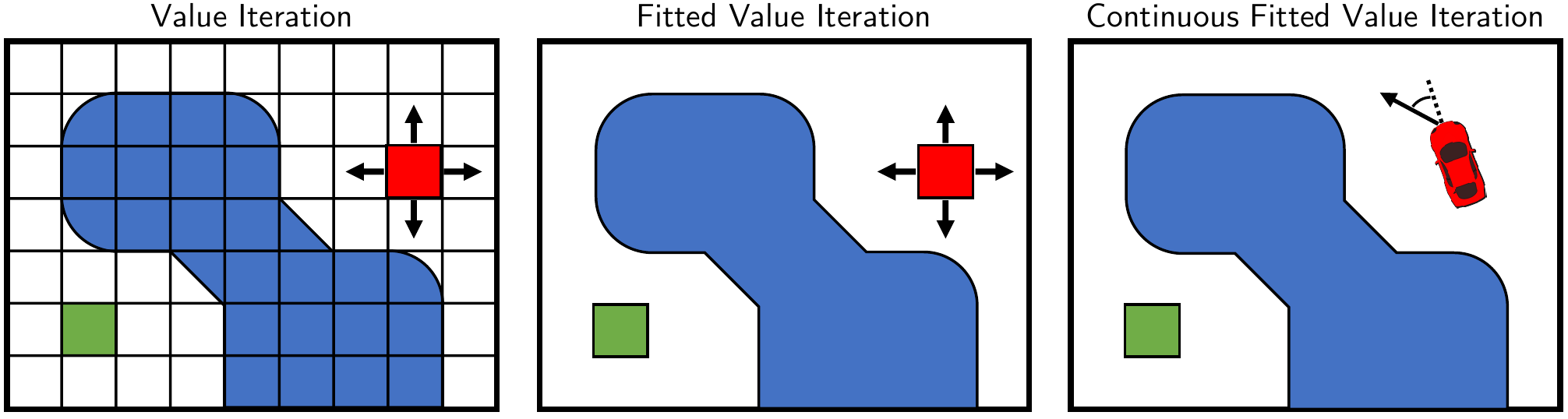}
    \vspace{-2.2em}
    \caption{Value iteration (VI) can compute the optimal value function $V^{*}$ and policy $\pi^{*}$ for discrete state and action environments. Fitted value iteration extends VI to continuous states and discrete actions. Our proposed continuous fitted value iteration enables value iteration for continuous states and continuous actions, e.g., a car with steering rather than a car moving left, right, up and down.}
    \label{fig:toy_figures}
    \vspace{-1.2em}
\end{figure} 

\begin{figure*}[t]
    \centering
    \includegraphics[width=\textwidth]{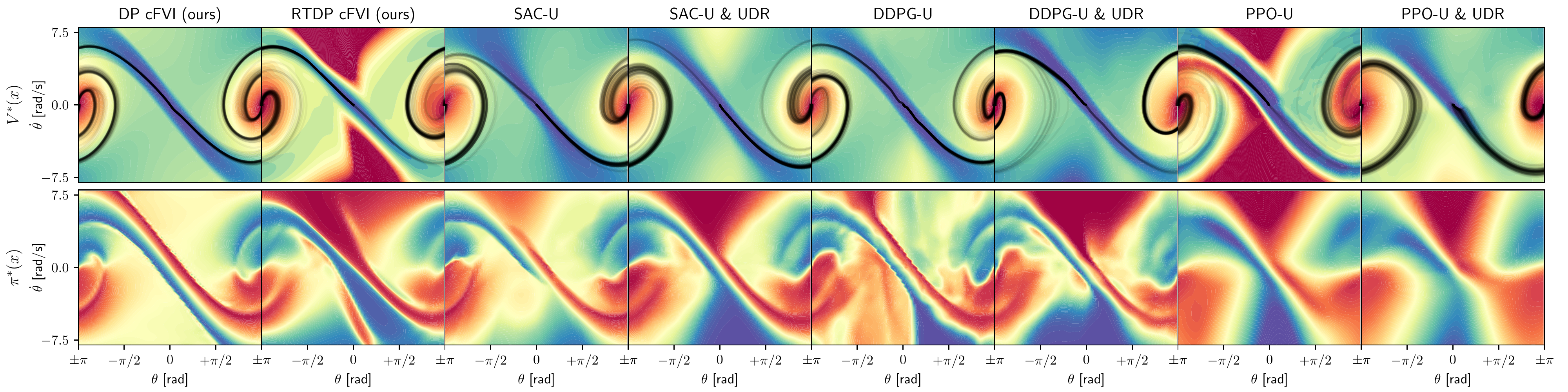}
    \vspace{-2.5em}
    \caption{The optimal value function $V^{*}$ and policy $\pi^{*}$ and roll outs of the torque-limited pendulum computed by DP cFVI, RTDP cFVI and the deep RL baselines with uniform initial state distribution. cFVI learns a symmetric and smooth optimal policy that can swing-up the pendulum from both sides. Especially the DP variant has a very sharp ridge leading to the upward pointing pendulum. The baselines do not achieve a symmetric swing-up and usually prefer a one-sided swing-up. %SAC outperforms the other deep RL variants due to the maximum entropy formulation and the resulting larger exploration.
    }
    \label{fig:value_fun_pendulum}
    \vspace{-1.2em}
\end{figure*} 

In this paper we show for the continuous-time RL that %\vspace{-0.5em}
\begin{enumerate} [wide=0pt]
\item classical value iteration \cite{bellman1957dynamic} can be extended to environments with continuous actions and states if the dynamics are control affine and the reward separable. Learning the value function $V$ is sufficient as the policy can be deduced from $V$ in the continuous-time limit. Therefore, one does not need the policy optimization used by the pre-dominant actor-critic approaches \cite{schulman2015high, lillicrap2015continuous} or the discretization required by the classical algorithms \cite{bellman1957dynamic, sutton1998introduction}.
\item the proposed approach can be successfully applied to learn the optimal policy for real-world under-actuated systems using only the approximate dynamics model. 
\item We provide an in-depth quantitative and qualitative evaluation with comparisons to actor-critic algorithms. Using value iteration on the compact state domain obtains a more robust policy when transferred to the physical systems.
\end{enumerate}
\textbf{Summary of contributions.} We extend value iteration to continuous actions as well as the extensive quantitative and qualitative evaluation on two physical systems. The latter is in contrast to prior work \cite{schulman2017proximal, haarnoja2018soft, lillicrap2015continuous,harrison2017adapt,mandlekar2017arpl}, which mainly focuses on the quantitative comparison in simulation. Instead, we focus on the qualitative evaluation on the physical Furuta pendulum and Cartpole to understand the differences in performance.   

In the following, we summarize continuous-time RL (Section~\ref{sec:problem}) and introduce continuous fitted value iteration (Section~\ref{sec:cont_rl}). Section \ref{sec:value} describes suitable network representations for a Lyapunov value function. Finally, we analyze the performance on the physical systems in Section \ref{sec:experiments}.      

\vspace{-0.5em}
\section{Problem Statement} 
\label{sec:problem} \vspace{-0.5em}
We consider the deterministic, continuous-time RL problem. The infinite horizon optimization is described by
\begin{gather}
\pi^*(\vx_0) = \argmax_{\pi} \int_{0}^{\infty} \exp(-\rho t) \:\: r_{c}(\vx_t, \vu_t) \: dt  \label{eq:cont_policy} \\ 
V^{*}(\vx_0) = \max_{\vu} \int_{0}^{\infty} \exp(-\rho t) \: r_c(\vx_{t}, \vu_{t}) \:  dt  \label{eq:cont_val} \\
\text{with} \hspace{10pt} \vx(t) = \vx_0 + \int_{0}^{t} f_c(\vx_{\tau}, \vu_{\tau}) \: d\tau 
\end{gather}
with the discounting constant $\rho \in (0, \infty]$, the reward $r_c$ and the dynamics $f_c$ \cite{kirk2004optimal}. Notably, the discrete reward and discounting can be described using the continuous-time counterparts, i.e., $r(\vx, \vu) = \Delta t \: r_c(\vx, \vu)$ and $\gamma = \exp(-\rho \: \Delta t)$ with the sampling interval $\Delta t$. The continuous-time discounting $\rho$ is, in contrast to the discrete discounting factor $\gamma$, agnostic to sampling frequencies. In the continuous-time case, the Q-function does not exist \cite{doya2000reinforcement}.

% Assumptions:
The deterministic continuous-time dynamics model $f_c$ is assumed to be non-linear w.r.t. the system state $\vx$ but affine w.r.t. the action~$\vu$. Such dynamics model is described by
\begin{align}
\dot{\vx} = \va(\vx; \theta) + \mB(\vx; \theta) \vu, \label{eq:affine_dyn} 
\end{align}
with the non-linear drift $\va$, the non-linear control matrix $\mB$ and the system parameters $\theta$. Robot dynamics models are naturally expressed in the continuous-time formulation and many are control affine. Furthermore, this special case has received ample attention in the existing literature due to its wide applicability \cite{doya2000reinforcement, kappen2005linear, todorov2007linearly}. 
The reward is separable into a non-linear state reward $q_c$ and the action cost $g_c$ described by
\begin{align}
    r_c(\vx, \vu) = q_c(\vx) - g_c(\vu). \label{eq:seperable_rwd}
\end{align}
The action penalty $g_{c}$ is non-linear, positive definite and strictly convex. This separability is common for robot control problems as rewards are composed of a state component quantifying the distance to a desired state % $\vx_{\text{des}}$ 
and an action penalty. The action cost penalizes non-zero actions to avoid bang-bang control from being optimal and is convex to have an unique optimal action.

\vspace{-0.5em}
\section{Continuous Fitted Value Iteration} 
\label{sec:cont_rl}\vspace{-0.5em}
First, we summarize the existing value iteration approaches and present the theory enabling us to extend value iteration to continuous action spaces. Afterwards, we introduce our proposed algorithm to learn the value function.

\vspace{-0.5em}
\subsection{Value Iteration Preliminaries} \vspace{-0.5em}
Value iteration (VI) is an approach to compute the optimal value function for a discrete time,  state and action MDP with known dynamics and reward \cite{bellman1957dynamic}. This approach iteratively updates the value function of each state using the Bellman optimality principle described by
\begin{align*}
 V^{k+1}(\vx_t) = \max_{\vu_{0 \dots \ell}} \: \sum_{i=0}^{\ell-1} \gamma^{i} r(\vx_{t + i}, \vu_{i}) + \gamma^\ell V^{k}(\vx_{t + \ell}),
\end{align*}
with the number of look-ahead steps $\ell$. VI is proven to converge to the optimal value function \cite{puterman1994markov} as the update is a contraction described by
\begin{align} \label{eq:contraction}
    \| V^{k+1}_{i} - V^{k+1}_{j} \| \leq \gamma^{\ell} \| V_i - V_j \|. 
\end{align}
For discrete actions, the greedy action is determined by evaluating all possible actions. However, VI is impractical for larger MDPs as the computational complexity grows exponentially with the number of states and actions. This is especially problematic for continuous spaces as discretization leads to an exponential increase of state-action tuples.

VI can be applied to continuous state spaces and \textbf{discrete} actions by using a function approximator instead of tabular value function. 
Fitted value iteration (FVI) computes the value function target using the VI update and minimizes the $\ell_p$-norm between the target and the approximation $V^{k}(\vx; \:\psi)$. This approach is described by
\begin{gather}
V_{\text{tar}}(\vx_t) = \max_{\vu} \: r(\vx_{t}, \vu) + \gamma V^{k}(f(\vx_{t}, \vu); \: \psi_k) \label{eq:fvi_update} \\ 
\psi_{k+1} = \argmin_{\psi} \sum_{\vx \in \mathcal{D}} \| V_{\text{tar}}(\vx) - V^{k}(\vx; \:\psi)  \|_{p}^{p} \label{eq:fitting}
\end{gather}
with the parameters $\psi_k$ at iteration $k$ and the fixed dataset~$\mathcal{D}$.
The convergence proof of VI does not generalize to FVI as the fitting of the value function is not necessarily a contraction \cite{boyan1994generalization, baird1995residual, tsitsiklis1996feature, munos2008finite}. 
Despite these theoretical limitations, many authors proposed model-free variants using the Q-function for discrete action MDPs, e.g., fitted Q-iteration (FQI) \cite{ernst2005tree}, Regularized FQI \cite{massoud2009regularized}, Neural FQI \cite{riedmiller2005neural} and DQN \cite{mnih2015human}. 
The resulting algorithms were empirically successful and solved Backgammon \cite{tesauro1992practical} and the Atari games \cite{mnih2015human}. 

For continuous actions, current RL methods use policy iteration (PI) rather than VI \cite{schulman2015high, lillicrap2015continuous,}. PI evaluates the value function of the current policy and hence, uses the action of the policy to compute the value function target. Therefore, PI circumvents the maximization required of VI (\Eqref{eq:fvi_update}). To update the policy, PI uses an additional optimization to improve the policy via policy gradients. Therefore, PI requires an additional optimization compared to VI approaches.

\begin{algorithm}[t]
\caption{Continuous Fitted Value Iteration (cFVI)}
\label{alg:cFVI}
\begin{algorithmic}
\STATE {\bfseries Input:} Dynamics Model $f_{c}(\vx, \vu)$ \& Dataset $\mathcal{D}$
\STATE {\bfseries Result:} Value Function $V^{*}(\vx;\: \psi^{*})$
\STATE
\FOR{k = 0 \dots N}
\STATE // Compute Value Target for $\vx \in \mathcal{D}$:\;
\STATE $V_{\text{tar}}(\vx_i) = \int_{0}^{T} \beta \: \exp(-\beta t) \: R_t \: dt + \exp(-\beta T) R_T$
\STATE $R_t = \int_{0}^{t} \exp(-\rho \tau) \: r_c(\vx_{\tau}, \vu_{\tau}) d\tau + \exp(-\rho t) V^{k}(\vx_t)$
\STATE $\vx_{\tau} = \vx_i + \int_{0}^{\tau} f_c(\vx_{t}, \vu_{t}) dt$
\STATE $\vu_{\tau} = \nabla \tilde{g}\left(\mB(\vx_{\tau}) \nabla_{x}V^{k}(\vx_{\tau}))\right)$  
\STATE
\STATE // Fit Value Function:
\STATE $\psi_{k+1} = \argmin_{\psi} \sum_{\vx \in \mathcal{D}} \| V_{\text{tar}}(\vx) - V(\vx; \psi) \|^{p}$ \;
\STATE
\IF{RTDP cFVI}
\STATE // Add samples from $\pi^{k+1}$ to FIFO buffer $\mathcal{D}$
\STATE $\mathcal{D}^{k+1} = h(\mathcal{D}^{k}, \{\vx^{k+1}_0 \: \dots \: \vx^{k+1}_N \})$
\ENDIF
\ENDFOR
\end{algorithmic}
\end{algorithm}

\vspace{-0.5em}
\subsection{Deriving the Analytic Optimal Policy}\vspace{-0.5em}
One cannot directly apply FVI to continuous actions due to the maximization in \Eqref{eq:fvi_update}. To compute the value function target, one would need to solve an optimization in each iteration. To extend value iteration to continuous actions, we show that one can solve this maximization analytically for the considered continuous-time problem. This solution enables the efficient computation of the value function target and extends FVI to continuous actions. 

\begin{theorem} \label{theorem:opt_policy}
If the dynamics are control affine (\Eqref{eq:affine_dyn}), the reward is separable w.r.t. to state and action (\Eqref{eq:seperable_rwd}) and the action cost $g_c$ is positive definite and strictly convex, the continuous-time optimal policy $\pi^{k}$ is described by
\begin{gather}
    \pi^{k}(\vx)  = \nabla \tilde{g}_c \left( \mB(\vx)^{T} \nabla_{x} V^{k} \right) \label{eq:theorem}
\end{gather}
where $\tilde{g}$ is the convex conjugate of $g$ and $\nabla_{x} V^k$ is the Jacobian of current value function $V^{k}$ w.r.t. the system state.
\end{theorem} 

\begin{hproof}
The detailed proof is provided in the appendix. This derivation follows the work of Lutter et. al. \yrcite{lutter2019hjb}, which generalized the special cases initially described by Lyshevski \yrcite{lyshevski1998optimal} and Doya \yrcite{doya2000reinforcement} to a wider class of reward functions. The value iteration target (\Eqref{eq:fvi_update}) can be expressed by using the Taylor series for the value function $V(\vx_{t+1})$ and substituting \Eqref{eq:affine_dyn} and \Eqref{eq:seperable_rwd}. This reformulated value target is described by
\begin{align*}
V_{\text{tar}}  % &= \max_{\vu} \:\: r + \gamma V + \gamma V_x^{T} f_c  \Delta t + \gamma \mathcal{O} \Delta t\\
                &= \max_{\vu} \:\: \left[ \gamma \nabla_{x}V^{T} \left(\va + \mB \vu \right) + \gamma \: \mathcal{O}(.) + q_c - g_c \right] \Delta t + \gamma V
\end{align*}
with the higher order terms $\mathcal{O}(\Delta t, \vx, \vu)$. 
In the continuous-time limit the higher order terms disappear and $\gamma = 1$. Therefore, the optimal action is described by 
\begin{align}
    \vu^{*}_{t} = \argmax_{\vu} \: \nabla_{x} V^{T} \mB(\vx_t) \: \vu - g_{c}(\vu). \label{eq:final_opt}
\end{align}
\Eqref{eq:final_opt} can be solved analytically as $g_c$ is strictly convex and hence $\nabla g_{c}(\vu) = \vw$ is invertible, i.e., $\vu = \left[ \nabla g_{c}\right]^{-1}(\vw) \coloneqq \nabla \tilde{g}_{c}(\vw)$ with the convex conjugate $\tilde{g}$. The solution of \Eqref{eq:final_opt} is described by
\begin{align*}
\mB^T \nabla_{x}V^{*} - \nabla g_c(\vu) \coloneqq 0 \hspace{5pt} \Rightarrow \hspace{5pt} \vu^{*} = \nabla \tilde{g}_c \left( \mB^{T} \nabla_{x}V^{k} \right).
\end{align*}
\end{hproof} 
\vspace{-1.5em}
Unrolling the closed form policy over time performs hill-climbing on the value function, where the step-size corresponds to the time-discretization. The inner part $\mB(\vx)^{T} \nabla_{x}V^{k}$ determines the direction of steepest ascent and the action cost rescales this direction. For example, a zero action cost for all admissible actions makes the optimal policy take the largest possible actions, i.e., the policy is bang-bang controller. % described by $\vu^* = \sign(\mB(\vx)^{T} V^{*}_{x}) \: \vu_{\text{max}}$. 
A quadratic action cost linearly re-scales the gradient. %, i.e., $\vu^* = \mG^{-1} \mB(\vx)^{T} V^{*}_{x}$. A 
A barrier shaped cost clips the gradients. A complete guide on designing the action cost to shape the optimal policy can be found in Lutter et. al. \yrcite{lutter2019hjb}. The step-size, which corresponds to the control frequency of the discrete time controller, determines the convergence to the value function maximum. If the step size is too large the system becomes unstable and does not converge to the maximum. For most real world robotic systems with natural frequencies below $5$Hz and control frequencies above $100$Hz, the resulting step-size is sufficiently small to achieve convergence to the value function maximum. Therefore, $\pi^{*}$ can be used for discrete time controllers. Furthermore, the continuous-time policy can be used for intermittent control (event-based control) where interacting with the system occurs at irregular time-steps and each interaction is associated with a cost \cite{astrom2008event}.  

\vspace{-0.5em}
\subsection{Learning the Optimal Value Function} 
\label{sec:cFVI} \vspace{-0.5em}
Using the closed-form policy that analytically solves the maximization, one can derive a computationally efficient value iteration update for continuous states and actions. Substituting $\vu^{*}$ (\Eqref{eq:theorem}) into the value iteration target computation (\Eqref{eq:fvi_update}) yields 
\begin{gather*}
    V_{\text{tar}}(\vx_t) = r\left(\vx_{t}, \nabla \tilde{g}\left(\mB(\vx_t)^{T} \nabla_{x}V^{k}\right)\right) + \gamma V^{k}(\vx_{t+1}; \: \psi_k) \\
    \text{with} \hspace{10pt} \vx_{t+1} = f\left(\vx_{t}, \nabla \tilde{g}(\mB(\vx_t)^{T} \nabla_{x}V^{k}) \right).
\end{gather*}
Combined with fitting the value function to the target value (\Eqref{eq:fitting}), 
these two steps constitute the value function update of cFVI. Repeatedly computing the value target and fitting the value function leads to the optimal value function and policy.

\begin{figure}[t]
    \centering
    \includegraphics[width=\columnwidth]{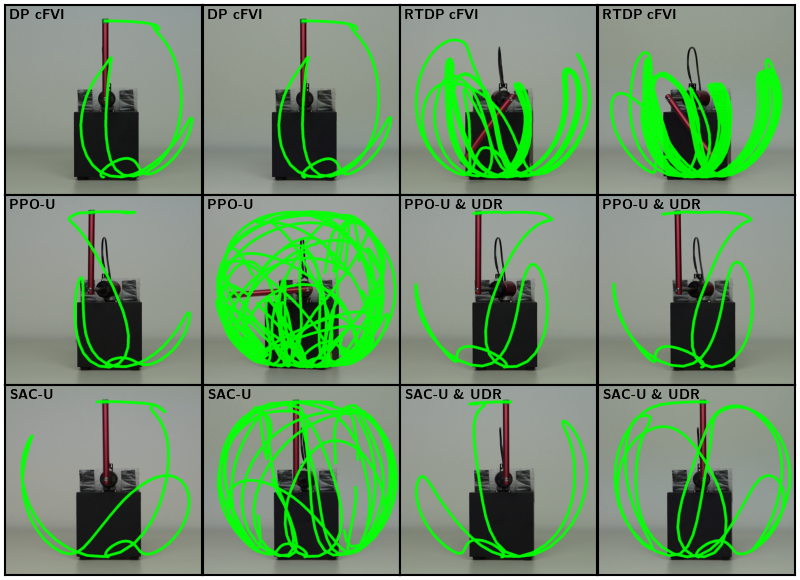}
    \vspace{-2.em}
    \caption{Tracked swing-up of the Furuta pendulum for best and worst roll out. DP cFVI can consistently swing-up the pendulum with minimal variance between roll outs. DP cFVI achieves much better performance compared to most baselines. Only PPO with domain randomization achieves comparable results. The remaining baselines are shown in the appendix.}
    \label{fig:furuta}\vspace{-1.2em}
\end{figure} 

For the continuous-time limit, the computation of the \emph{naive} value function target should be adapted as the convergence decreases with decreasing time steps. As $\Delta t$ decreases, the $\gamma$ increases, i.e., $\gamma = \lim_{\Delta t \rightarrow 0} \exp(-\rho \Delta t) = 1$. Therefore, the contraction coefficient of VI decreases exponentially with increasing sampling frequencies. This slower convergence is intuitive as higher sample frequencies effectively increase the number of steps to reach the goal. 

\noindent \textbf{$\mathbf{N}$-Step Value Function Target}
To increase the computational efficiency cFVI, the exponentially weighted n-step value function target
\begin{gather*}
V_{\text{tar}}(\vx) = \int_{0}^{T} \beta \: \exp(-\beta t) \: R_t \: dt + \exp(-\beta T) R_T,  \\ 
% \text{with} \hspace{5pt}
R_t = \int_{0}^{t} \exp(-\rho \tau) \: r_c(\vx_{\tau}, \vu_{\tau}) d\tau + \exp(-\rho t) V^{k}(\vx_t), \label{eq:R_t}
\end{gather*}
and the exponential decay constant $\beta$, can be used. The integrals can be solved using any ordinary differential equation solver with fixed or adaptive step-size. We use the explicit Euler integrator with fixed steps to solve the integral for all samples in parallel using batched operations on the GPU. The nested integrals can be computed efficiently by recursively splitting the integral and reusing the estimate of the previous step. In practice we treat $\beta$ as a hyperparameter and select $T$ such that the weight of the $R_T$ is $\exp\left(-\beta T\right) \coloneqq 10^{-4}$.

The convergence rate improves as this target computation corresponds to the multi-step value target. The n-step target increases the contraction rate as this rate is proportional to $\gamma^{\ell}$. This approach is the continuous-time counterpart of the discrete eligibility trace of TD($\lambda$) with $\lambda = \exp(- \beta \Delta t)$ \citep{sutton1998introduction}. With respect to deep RL this discounted n-step value target is similar to the generalized advantage estimation (GAE) of PPO \cite{schulman2015high, schulman2017proximal} and model-based value expansion (MVE) \cite{feinberg2018model, buckman2018sample}. GAE and MVE have shown that the $n$-step target increases the sample efficiency and lead to faster convergence to the optimal policy.

\textbf{Offline and Online cFVI}
The proposed approach is off-policy as the samples in the replay memory do not need to originate from the current policy $\pi_k$. Therefore, the dataset can either consist of a fixed dataset (i.e., batch or offline RL) or be updated within each iteration. In the offline dynamic programming case, the dataset contains samples from the compact state domain~$\mathcal{X}$. We refer to the offline variant as DP cFVI. In the online case, the replay memory is updated with samples generated by the current policy $\pi_k$. Every iteration the states of the previous $n$-rollouts are added to the data and replace the oldest samples. 
This online update of state distribution performs real-time dynamic programming (RTDP)~\cite{barto1995learning}. We refer to the online variant as RTDP cFVI. The pseudo code of DP cFVI and RTDP cFVI is summarized in Algorithm~\ref{alg:cFVI}.

\vspace{-0.5em}
\section{Value Function Hypothesis Space} 
\label{sec:value} \vspace{-0.5em}
The previous sections focused on learning the optimal value function independent of the value function representation. In this section, we focus on the value function representation. 

\textbf{Value Function Representation}
Most recent approaches use a fully connected network to approximate the value function. However, for the many tasks the hypothesis space of the value function can be narrowed. For control tasks, the state cost is often a negative distance measure between $\vx_t$ and the desired state $\vx_{\text{des}}$. Hence, $q_c$ is negative definite, i.e., $q(\vx) < 0 \:\: \forall \:\: \vx \neq \vx_{\text{des}}$ and $q(\vx_{\text{des}}) = 0$. 
These properties imply that $V^{*}$ is a negative Lyapunov function, as $V^{*}$ is negative definite, $V^{*}(\vx_{\text{des}}) = 0$ and $\nabla_{x}V^{*}(\vx_{\text{des}}) = \mathbf{0}$ \cite{khalil2002nonlinear}. With a deep network a similar representation can be achieved by
\begin{align*}
    % V(\vx; \: \psi) &= f(\vx; \:\psi) - f(\vx_{\text{tar}}; \:\psi) - \frac{\partial f(\vx_{\text{tar}}; \:\psi)}{\partial \vx} \vx \\
    % V_x(\vx; \: \psi) &= \frac{\partial f(\vx; \:\psi)}{\partial \vx} - \frac{\partial f(\vx_{\text{tar}}; \:\psi)}{\partial \vx}
    V(\vx; \: \psi) &= -\left(\vx -  \vx_{\text{des}}\right)^T \mL(\vx;\:\psi) \mL(\vx;\:\psi)^T \left(\vx -  \vx_{\text{des}}\right) % \\
\end{align*}
with $\mL$ being a lower triangular matrix with positive diagonal. This positive diagonal ensures that $\mL \mL^T$ is positive definite. Simply applying a ReLu activation to the last layer of a deep network is not sufficient as this would also zero the actions for the positive values and $\nabla_{x}V^{*}(\vx_{\text{des}}) = \mathbf{0}$ cannot be guaranteed. The local quadratic representation guarantees that the gradient and hence, the action, is zero at the desired state. However, this representation can also not guarantee that the value function has only a single extrema at~$\vx_{\text{des}}$ as required by the Lyapunov theory. In practice, the local regularization of the quadratic structure to avoid high curvature approximations is sufficient as the global structure is defined by the value function target. $\mL$ is the mean of a deep network ensemble with $N$ independent parameters $\psi_i$. The ensemble mean smoothes the initial value function and is differentiable. Similar representations have been used by prior works in the safe reinforcement learning community \cite{berkenkamp2017safe, richards2018lyapunov, kolter2019learning, chang2019neural,bharadhwaj2021csc}. 

\textbf{Gradient Projection of State Transformations}
State transformations should be explicitly incorporated in the value function and should not be contained in the environment, as for example in the openAI Gym \cite{brockman2016gym}. If the transformation is implicit, $\nabla_{x}V$ might not be sensible. For example, the standard feature transform for a continuous revolute joint maps the joint state $\vx = [ \theta,\: \dot{\theta}]$ to $\vz = [\sin(\theta), \: \cos(\theta), \: \dot{\theta} ]$. In this case the transformation $h(\vx)$ must be included in the value function as the transformed state $\vz$ is on the tube shaped manifold 
%$\mathcal{M} \subset \mathcal{R}^{n}$
. Hence, the naive gradient of $V$ might not be in the manifold tangent space. % i.e., $\nabla_{z}V \notin \mathcal{T}_z \mathcal{M} \hspace{3pt} \forall \hspace{3pt} \vz \in \mathcal{M}$. 

\noindent
If the state transform is explicitly incorporated in the value function, this problem does not occur. The transformation can be included explicitly by $V(x;\: \psi) = f(h(\vx); \: \psi)$ and $\nabla_{x}V(x;\: \psi) = \partial f(h(\vx); \: \psi)/ \partial h \:\: \partial h(\vx)/\partial \vx$. 
In this case, the gradient of the transformed state is projected to the tangent space of the feature transform. Therefore, the value function gradient points in a plausible direction.

\begin{figure}[t]
    \centering
    \includegraphics[width=\columnwidth]{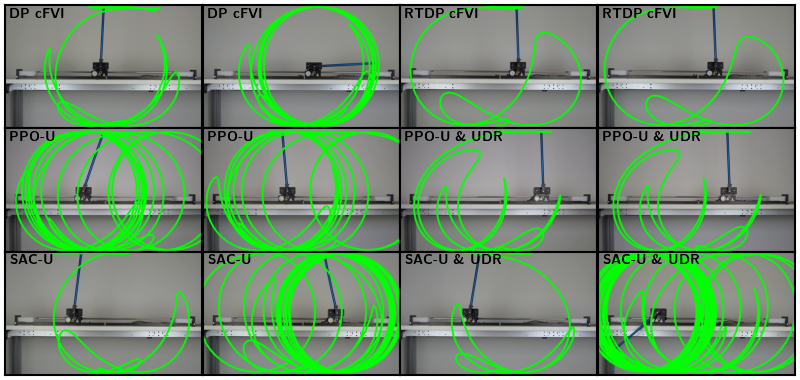}
    \vspace{-2.em}
    \caption{Tracked swing-up of the cartpole for the best and worst roll out. DP cFVI and RTDP cFVI can achieve the task. In the failure case, cFVI excels as the policy remains in the center. In contrast the deep RL baselines move the cart between the joint limits in case of failure. All baselines are shown in the appendix.}
    \label{fig:cartpole}\vspace{-1.2em}
\end{figure}

\begin{figure*}[t]
    \centering
    \includegraphics[width=\textwidth]{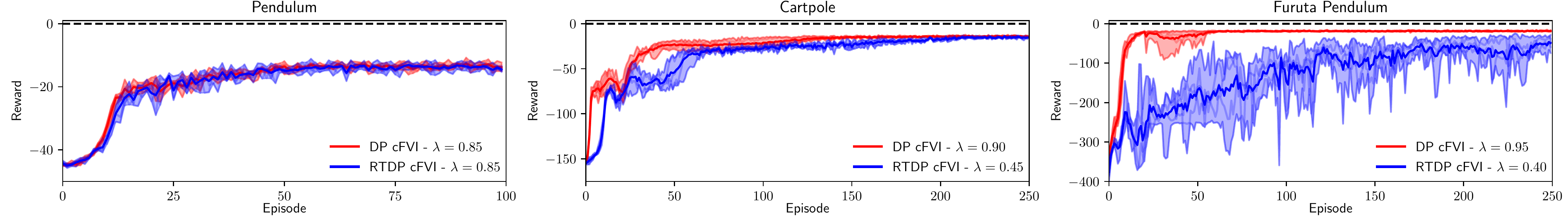}
    \vspace{-2.em}
    \caption{The learning curves for DP cFVI and RTDP cFVI averaged over $5$ seeds. The shaded area displays the \emph{min/max} range between seeds. DP cFVI achieves consistent learning of the optimal policy with very low variance between seeds. RTDP cFVI has a higher variation and learns slower compared to DP cFVI. RTDP cFVI needs to discover the steady-state state distribution first. Especially for the Furuta pendulum the range increases between seeds as minor changes in the policy lead to large changes of state-distribution.}
    \label{fig:learning_curves}\vspace{-0.8em}
\end{figure*} 

\vspace{-0.5em}
\section{Experiments} \label{sec:experiments}
\vspace{-0.5em}
In the following the experimental setup and results are described. The exact experiment specification, all qualitative plots and an additional ablation study on model ensembles is provided in the appendix. 

\begin{table*}[t]
\tiny
% \scriptsize
\centering
\renewcommand{\arraystretch}{1.1}
\caption{\footnotesize
Average rewards on the simulated and physical systems. The ranking describes the decrease in reward compared to the best result averaged on all systems. 
% The three best samples (\textcolor{red}{\textbf{1st}}, \textcolor{blue}{\textbf{2nd}} \& \textcolor{Orange}{\textbf{3rd}}) are highlighted. 
The initial state distribution during training is noted by~$\mu$. The dynamics are either deterministic model $\theta \sim \delta(\theta)$ or sampled using uniform domain randomization $\theta \sim \mathcal{U}(\theta)$.
During evaluation the roll outs start with the pendulum pointing downwards. DP cFVI obtains high rewards on all systems and is the highest  ranking algorithm compared to the baselines.
}
\setlength{\tabcolsep}{3.2pt}
\begin{tabular*}{\textwidth}{l c c c c c c c c | c c  c c | c}
% \begin{tabular*}{\textwidth}{l r r | c c c | c c c}
\toprule
 & & & \multicolumn{2}{c}{\textbf{Simulated Pendulum}}  & \multicolumn{2}{c}{\textbf{Simulated Cartpole}} & \multicolumn{2}{c|}{\textbf{Simulated Furuta Pendulum}}   & \multicolumn{2}{c}{\textbf{Physical Cartpole}} & \multicolumn{2}{c|}{\textbf{Physical Furuta Pendulum}} & \textbf{Average} \\
%  \cmidrule(lr){4-6} \cmidrule(lr){7-9}
& & & % \multicolumn{1}{c}{Algorithm} & $\mu$ & $\vtheta$ &  
Success & Reward & Success & Reward & Success & Reward  & Success & Reward & Success & Reward  & \textbf{Ranking}
\\  
 % \multicolumn{1}{c}{Algorithm} 
 \multicolumn{1}{c}{Algorithm} & $\mu$ & $\theta$ &   [$\%$] & [$\mu \pm 1.96 \sigma$] & [$\%$] & [$\mu \pm 1.96 \sigma$] & [$\%$] & [$\mu \pm 1.96 \sigma$] & [$\%$] & [$\mu \pm 1.96 \sigma$] & [$\%$] & [$\mu \pm 1.96 \sigma$] & [$\%$] \\
 \cmidrule(lr){1-3} \cmidrule(lr){4-5} \cmidrule(lr){6-7} \cmidrule(lr){8-9} \cmidrule(lr){10-11} \cmidrule(lr){12-13} \cmidrule(lr){14-14} 
DP cFVI (ours) & $\mathcal{U}$ & $\mathcal{U}(\theta)$ 
& $100.0$ &\textcolor{black}{$\mathbf{-030.5 \pm 000.8}$}
& $100.0$ &\textcolor{black}{$\mathbf{-024.2 \pm 002.1}$}
& $100.0$ &\textcolor{black}{$\mathbf{-027.7 \pm 001.6}$}
& $73.3$ & $-143.7 \pm 210.4$ 
& $100.0$ &\textcolor{black}{$\mathbf{-082.1 \pm 007.6}$}
& \textcolor{black}{$\mathbf{-008.8}$}
\\
RTDP cFVI (ours) & $\mathcal{U}$ & $\mathcal{U}(\theta)$ 
& $100.0$ &\textcolor{black}{$\mathbf{-031.1 \pm 001.4}$}
& $100.0$ &\textcolor{black}{$\mathbf{-024.9 \pm 001.6}$}
& $100.0$ & $-040.1 \pm 002.7$ 
& $100.0$ &\textcolor{black}{$\mathbf{-101.1 \pm 029.0}$}
& $00.0$ & $-1009.9 \pm 004.5$ 
& \textcolor{black}{$-240.4$}
\\
\cmidrule(lr){1-3} \cmidrule(lr){4-5} \cmidrule(lr){6-7} \cmidrule(lr){8-9} \cmidrule(lr){10-11} \cmidrule(lr){12-13} \cmidrule(lr){14-14}
SAC & $\mathcal{N}$ & $\mathcal{U}(\theta)$ 
& $100.0$ &\textcolor{black}{$\mathbf{-031.1 \pm 000.1}$}
& $100.0$ & $-026.9 \pm 003.2$ 
& $100.0$ & $-029.3 \pm 001.5$ 
& $00.0$ & $-518.6 \pm 028.1$ 
& $86.7$ & $-330.7 \pm 799.0$ 
& \textcolor{black}{$-148.4$}
\\
SAC \& UDR & $\mathcal{N}$ & $\delta(\theta))$ 
& $100.0$ & $-032.9 \pm 000.6$ 
& $100.0$ & $-029.7 \pm 004.6$ 
& $100.0$ & $-032.0 \pm 001.1$ 
& $100.0$ & $-394.8 \pm 382.8$ 
& $100.0$ & $-181.4 \pm 157.9$ 
& \textcolor{black}{$-092.3$}
\\
SAC & $\mathcal{U}$ & $\mathcal{U}(\theta)$ 
& $100.0$ &\textcolor{black}{$\mathbf{-030.6 \pm 001.4}$}
& $100.0$ &\textcolor{black}{$\mathbf{-024.2 \pm 001.4}$}
& $100.0$ &\textcolor{black}{$\mathbf{-028.1 \pm 002.0}$}
& $53.3$ & $-144.5 \pm 204.0$ 
& $100.0$ & $-350.8 \pm 433.3$ 
& \textcolor{black}{$-076.1$}
\\
SAC \& UDR & $\mathcal{U}$ & $\mathcal{U}(\theta)$ 
& $100.0$ & $-031.4 \pm 002.5$ 
& $100.0$ &\textcolor{black}{$\mathbf{-024.2 \pm 001.3}$}
& $100.0$ &\textcolor{black}{$\mathbf{-028.1 \pm 001.3}$}
& $40.0$ & $-296.4 \pm 418.9$ 
& $100.0$ & $-092.3 \pm 064.1$ 
& \textcolor{black}{$-042.4$}
\\
\cmidrule(lr){1-3} \cmidrule(lr){4-5} \cmidrule(lr){6-7} \cmidrule(lr){8-9} \cmidrule(lr){10-11} \cmidrule(lr){12-13} \cmidrule(lr){14-14}
DDPG & $\mathcal{N}$ & $\mathcal{U}(\theta)$ 
& $100.0$ &\textcolor{black}{$\mathbf{-031.1 \pm 000.4}$}
& $98.0$ & $-050.4 \pm 285.6$ 
& $100.0$ & $-030.5 \pm 003.5$ 
& $06.7$ & $-536.7 \pm 262.7$ 
& $46.7$ & $-614.1 \pm 597.8$ 
& \textcolor{black}{$-242.6$}
\\
DDPG \& UDR & $\mathcal{N}$ & $\delta(\theta))$ 
& $100.0$ & $-032.5 \pm 000.5$ 
& $100.0$ & $-027.4 \pm 002.3$ 
& $100.0$ & $-034.6 \pm 009.8$ 
& $00.0$ & $-517.9 \pm 117.6$ 
& $86.7$ & $-192.7 \pm 404.8$ 
& \textcolor{black}{$-119.3$}
\\
DDPG & $\mathcal{U}$ & $\mathcal{U}(\theta)$ 
& $100.0$ & $-031.5 \pm 000.7$ 
& $100.0$ & $-028.2 \pm 005.5$ 
& $100.0$ & $-030.0 \pm 001.7$ 
& $06.7$ & $-459.4 \pm 248.3$ 
& $100.0$ & $-146.6 \pm 218.3$ 
& \textcolor{black}{$-092.9$}
\\
DDPG \& UDR & $\mathcal{U}$ & $\mathcal{U}(\theta)$ 
& $100.0$ & $-032.5 \pm 003.6$ 
& $100.0$ & $-027.2 \pm 001.0$ 
& $100.0$ & $-032.1 \pm 001.5$ 
& $00.0$ & $-318.1 \pm 063.4$ 
& $100.0$ & $-156.7 \pm 246.4$ 
& \textcolor{black}{$-068.8$}
\\
\cmidrule(lr){1-3} \cmidrule(lr){4-5} \cmidrule(lr){6-7} \cmidrule(lr){8-9} \cmidrule(lr){10-11} \cmidrule(lr){12-13} \cmidrule(lr){14-14}
PPO & $\mathcal{N}$ & $\mathcal{U}(\theta)$ 
& $100.0$ & $-032.0 \pm 000.2$ 
& $100.0$ & $-031.5 \pm 007.2$ 
& $100.0$ & $-081.1 \pm 018.3$ 
& $00.0$ & $-287.9 \pm 068.8$ 
& $33.3$ & $-718.7 \pm 456.1$ 
& \textcolor{black}{$-240.9$}
\\
PPO \& UDR & $\mathcal{N}$ & $\delta(\theta))$ 
& $100.0$ & $-032.3 \pm 000.6$ 
& $100.0$ & $-084.0 \pm 007.8$ 
& $100.0$ & $-040.9 \pm 004.6$ 
& $00.0$ & $-435.4 \pm 111.9$ 
& $46.7$ & $-935.7 \pm 711.6$ 
& \textcolor{black}{$-338.6$}
\\
PPO & $\mathcal{U}$ & $\mathcal{U}(\theta)$ 
& $100.0$ & $-033.4 \pm 004.7$ 
& $99.0$ & $-039.7 \pm 045.7$ 
& $100.0$ & $-038.2 \pm 013.1$ 
& $00.0$ & $-183.8 \pm 018.0$ 
& $60.0$ & $-755.3 \pm 811.0$ 
& \textcolor{black}{$-206.1$}
\\
PPO \& UDR & $\mathcal{U}$ & $\mathcal{U}(\theta)$ 
& $100.0$ & $-035.6 \pm 003.1$ 
& $100.0$ & $-044.8 \pm 021.4$ 
& $100.0$ & $-048.5 \pm 006.2$ 
& $40.0$ & $-143.8 \pm 016.1$ 
& $100.0$ &\textcolor{black}{$\mathbf{-080.6 \pm 010.8}$}
& \textcolor{black}{$-044.0$}
\\
\bottomrule
\end{tabular*}
\vspace{-2.5em}
\label{table:cFVI_results}
\end{table*}

\vspace{-0.5em}
\subsection{Experimental Setup}
\vspace{-0.5em}
\textbf{Systems} The algorithms are compared using the standard non-linear control benchmark, the \emph{swing-up} of under-actuated systems. Specifically, we apply the algorithms to the torque-limited pendulum, cartpole (Figure~\ref{fig:cartpole}) and Furuta pendulum (Figure~\ref{fig:furuta}). For the physical systems the dynamics model of the manufacturer is used \cite{quanser}. The control frequency is optimized for each algorithm. The task is considered solved, if the pendulum angle $\alpha$ is below $\pm 5^{\circ}$ degree for the last second.

\textbf{Baselines} The performance is compared to the actor-critic deep RL methods: DDPG \cite{lillicrap2015continuous}, SAC \cite{haarnoja2018soft} and PPO \cite{schulman2017proximal}. We compare two different initial state distributions $\mu$. For $\mu = \mathcal{U}$, the initial pendulum angle $\alpha_{0}$ is sampled uniformly $\alpha_{0} \sim \mathcal{U}(-\pi, +\pi)$. For $\mu = \mathcal{N}$, the initial angle is sampled from a Gaussian distribution with the pendulum facing downwards $\alpha_{0} \sim \mathcal{N}(\pm \pi, \sigma)$. The uniform sampling avoids the exploration problem and generates a larger state distribution of the optimal policy. In addition we augment each baseline with uniform domain randomization \cite{muratore2018domain}.

\begin{figure*}[t]
    \centering
    \includegraphics[width=\textwidth]{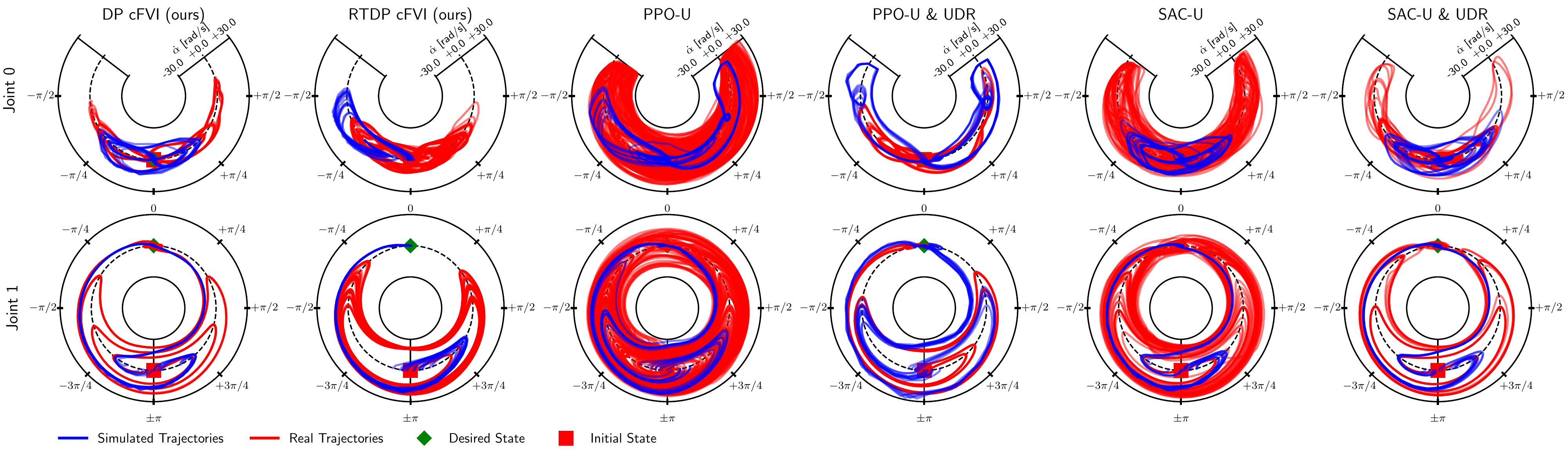}
    \vspace{-2.em}
    \caption{The simulated (blue) and real world (red) roll outs for the Furuta pendulum controlled by cFVI and the baselines. The deviation from the dashed line denotes the joint velocity. The distribution shift is clearly visible as the blue and red trajectories do not overlap, e.g., cFVI requires more pre-swings on the physical system. For PPO-U the distribution shift causes random exploration of the complete state-domain. cFVI achieves the best qualitative performance as it can swing-up the pendulum from both directions and achieves nearly identical roll-outs. In contrast the deep RL baselines have higher distribution shift due to the dynamics mismatch. The trajectories are also less consistent and vary significantly between roll outs. The plots for all baselines and the physical cartpole are shown in the appendix.}
    \label{fig:furuta_rollout}\vspace{-0.8em}
\end{figure*} 

\vspace{-0.5em}
\subsection{Simulation Results} \vspace{-0.5em}
The average rewards of our method - Continuous Fitted Value Iteration (cFVI) and the baselines are summarized in Table \ref{table:cFVI_results}. The learning curves are shown in Figure \ref{fig:learning_curves}. In simulation, DP cFVI, the offline version of cFVI with fixed dataset, achieves the best rewards for all three systems and marginally outperforms SAC in terms of average reward. It's important to point out that identical mean rewards do not imply a similar qualitative performance. For example, DP cFVI and SAC-U achieve identical mean reward on the cartpole despite having very different closed-loop dynamics (Figure~11 - Appendix). Notably, RTDP cFVI, which uses the state distribution of the current policy rather than a fixed dataset, solves the tasks for all three systems. For the pendulum and the cartpole, the reward is comparable to the best performing algorithms. While for the Furuta pendulum, the reward is lower (higher) than the best (worst) performing Deep RL baselines.

The learning curves highlight that the variance between seeds is very low for DP cFVI. It is important to note that Figure \ref{fig:learning_curves} shows the min-max range of the different seeds rather than the confidence intervals or the standard error. Therefore, the exact weight initialization and sampling of the fixed dataset, which is different for each seed, do not have a large impact on the learning progress or obtained solution. For RTDP cFVI the variance between seed increases and the learning speed decreases compared to DP cFVI. The pendulum is an outlier for RTDP cFVI, as the state-distribution is nearly independent of the current policy if the pendulum angle is initialized uniformly. Especially for the Furuta pendulum the variance increases and learning speed decreases. For this system slightly different policies can cause vastly different state distributions due to the low condition number of the mass-matrix. Therefore, RTDP cFVI takes longer to reach a stationary distribution and increase the obtained reward. 

In general it is important to point out that the exploration for RTDP cFVI is more challenging compared to other deep RL methods as RTDP cFVI uses a higher control frequency. Due to the shorter steps simple Gaussian noise averages out and does not lead to a large exploration in state space. Therefore, the performance and variance of RTDP cFVI could be improved using a more coherent exploration strategy in future work. 

\vspace{-0.5em}
\subsection{Ablation Study - $N$-step Value Targets}\vspace{-0.5em}
The learning curves, averaged over $5$ seeds, for different $n$-step value targets are shown in Figure \ref{fig:ablation_lambda}. When increasing~$\lambda$, which implicitly increases the $n$-step horizon (section \ref{sec:cFVI}), the convergence rate to the optimal value function increases. This increased learning speed is expected as \Eqref{eq:contraction} shows that the convergence rate depends on $\gamma^{n}$ with $\gamma < 1$. While the learning speed measured in iterations increases with $\lambda$, the computational complexity also increases. Longer horizons require to simulate $n$ sequential steps increasing the required wall-clock time. Therefore, the computation time increases exponentially with increasing $\lambda$. For example the forward roll out in every iteration of the pendulum increases exponentially from $0.4$s ($\lambda = 0.01$) % , $1.7$s ($\lambda = 0.5$), $15.0$s ($\lambda = 0.95$) 
to $56.4$s ($\lambda = 0.99$)\footnote{Wall-clock time on an AMD 3900X and a NVIDIA RTX 3090}. For the Furuta pendulum and the cartpole extremely long horizons of $100+$ steps start to diverge as the value function target over fits to the untrained value function. For RTDP cFVI, the horizons must smaller, i.e., $10$ - $20$ steps. For longer horizons the predicted rollout overfits to the value function outside of the current state distribution, which prevents learning or leads to pre-mature convergence (see Appendix Figure 9). Therefore, DP cFVI works bests with $\lambda \in [0.85, 0.95]$ and RTDP cFVI with $\lambda \in [0.45 - 0.55]$.

\vspace{-0.5em}
\subsection{Physical Experiments}\vspace{-0.5em}
The quantitative results of the physical experiments are summarized in Table \ref{table:cFVI_results}\footnote{Task Video available at: \url{https://sites.google.com/view/value-iteration}}, with additional plots and baselines in the appendix. DP cFVI achieves high reward on both physical systems and outperforms the baselines. Only PPO-U UDR achieves comparable performance on both physical systems. On average the performance of the deep RL variants increases with with larger initial state distribution.

The trajectories of the Furuta Pendulum are shown in Figure~\ref{fig:furuta_rollout}. The distribution shift between the physical and simulated trajectories is clearly visible as the trajectories largely deviate. DP cFVI achieves a highly repeatable performance on the noisy physical system. All 15 roll-outs have a similar trajectory. The baselines require domain randomization and uniform initial distribution to achieve a reward comparable to DP cFVI. All other deep RL baselines have a high variance between roll-outs. For example the simulation gap causes PPO-U to randomly explore the complete state space. 
RTDP cFVI does not achieve the swing-up but fails gracefully as it stabilizes the pendulum on a stable limit cycle and does not hit the joint limits. Furthermore, this limit cycle is highly repetitive and the variance in reward is low. For many robotic tasks, it is preferable to fail gracefully rather than hitting the limits and solving the task. 

\begin{figure*}[t]
    \centering
    \includegraphics[width=\textwidth]{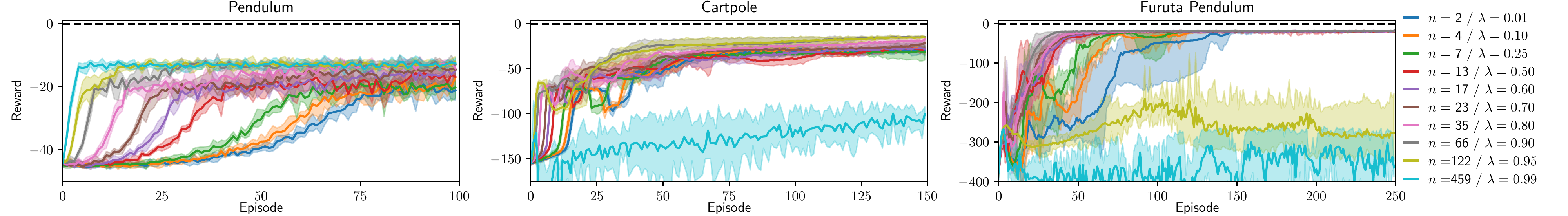}
    \vspace{-2.em}
    \caption{The learning curves averaged over $5$ seeds for the $n$-step value function target. 
    %for different exponential discounting constants $\lambda = \exp(-\beta \Delta t)$. 
    The shaded area displays the \emph{min/max} range between seeds. The step count is selected such that~$\lambda^{n} \coloneqq 10^{-4}$. Increasing the horizon of the value function target increases the convergence rate to the optimal value function. For very long horizons the learning diverges as it over fits to the current value function approximation. Furthermore, the performance of the optimal policy also increases with roll out length.
    }
    \label{fig:ablation_lambda} \vspace{-0.8em}
\end{figure*} 

On the physical cartpole, DP cFVI obtains a higher reward than RTDP cFVI, if the swing-up is successful. However, RTDP cFVI has a higher average reward across all roll-outs as DP cFVI only achieves a $73\%$ success rate. If DP cFVI fails, the resulting behavior is a deterministic limit-cycle where the cart is centered. The policy decelerates the pendulum, but due to the backslash in the linear actuator the applied force is not sufficient. 
Most deep RL baselines do not achieve a consistent swing-up. Only PPO-U UDR, SAC-U achieve a repeatable high reward. SAC-N UDR achieves the swing-up but always hits the joint limit and obtains a relative low reward. The failure cases of all deep RL baselines are stochastic and repeatedly hit the joint limit. If the pendulum over-swings the cart starts to oscillate between both joint limits. 
Notably, domain randomization does not improve the performance on the physical cartpole. For this system, the simulation gap originates from backslash and friction in the actuator rather than the uncertainty of the model parameters. As the actuation is not simulated, these parameters cannot be randomized. For the cartpole, the larger initial state distribution is required to achieve high rewards on the physical system. All baselines with a gaussian initial state distribution achieve only a low reward. See appendix for qualitative simulated and physical results. 

\vspace{-0.5em}
\section{Related Work}\vspace{-0.5em}
Continuous-time RL started with the seminal work of Doya \yrcite{doya2000reinforcement}. Since then, various approaches have been proposed to solve the Hamilton-Jacobi-Bellman (HJB) differential equation with the machine learning toolset. These methods can be divided into trajectory and state-space based methods. 

Trajectory based methods solve the stochastic HJB along a trajectory to obtain the optimal trajectory. For example path integral control uses the non-linear, control-affine dynamics with quadratic action costs to simplify the HJB to a linear partial differential equation~\cite{kappen2005linear, todorov2007linearly, theodorou2010reinforcement, pan2014model}. This differential equation can be transformed to a path-integral using the Feynman-Kac formulae. The path-integral can then be solved using Monte Carlo sampling to obtain the optimal state and action sequence. Recently, this approach has been used in combination with deep networks~\cite{rajagopal2016neural, pereira2019learning, pereira2020safe}.

State-space based methods solve the HJB globally to obtain a optimal non-linear controller applicable on the complete state domain. Classical approaches discretize the continuous spaces into a grid and solve the HJB or the robust Hamilton-Jacobi-Isaac (HJI) using a PDE solver \cite{bansal2017hamilton}. To overcome the curse of dimensionality of the grid based methods, machine learning methods that use function approximation and sampling have been proposed. For example, regression based approaches solved the HJB by fitting a radial-basis-function networks \cite{doya2000reinforcement}, deep networks \cite{tassa2007least, lutter2019hjb, kim2020hamilton}, kernels \cite{hennig2011optimal} or polynomial functions \citep{yang2014reinforcement, liu2014neural} to minimize the HJB residual. The naive optimization of this objective, while successful for different PDEs \cite{raissi2017physics_1, raissi2017physics_2, raissi2020hidden}, does not work for the HJB as the exact location of the boundary condition is unknown. Therefore, authors used various optimization tricks such as annealing the noise \cite{tassa2007least} or the discounting \cite{lutter2019hjb} to obtain the optimal value function.

In this work we propose a DP based algorithm. Instead of trying to solve the HJB PDE directly via regression as previous learning-based methods, we leverage the continuous-time domain and solve the HJB by iteratively applying the Bellman optimality principle. The resulting algorithm has a simpler and more robust optimization compared to the previous approaches using direct regression.   

\vspace{-0.5em}
\section{Conclusion and Future Work}\vspace{-0.5em}
We proposed continuous fitted value iteration (cFVI). This algorithm enables dynamic programming with known model for problems with continuous states and actions without discretization. Therefore, cFVI avoids the curse of dimensionality associated with discretization and the policy optimization of actor-critic approaches. Exploiting the insights from the continuous-time formulation, the optimal action can be computed analytically. This closed-form solution permits the efficient computation of the value function target and enables the extension of value iteration. 
The non-linear control experiments showed that value iteration on the compact state space has the same quantitative performance as deep RL methods in simulation. On the sim2real tasks, DP cFVI excels compared deep RL algorithms that include domain randomization and use uniform initial state distribution.

In future work, we plan to extend cFVI to (1) offline/batch model-based RL, (2) stochastic maximum entropy policies and (3) explicit model robustness. cFVI uses a fixed data set and hence, can be combined with learning control-affine models to the offline RL benchmark \cite{gulcehre2020rl,mandlekar2019roboturkreal}. 
We plan to extend cFVI to maximum entropy and stochastic policies to improve the exploration of RTDP cFVI. Finally, we plan to incorporate robustness w.r.t. to changes in dynamics by using the adversarial continuous RL formulation culminating in the Hamilton-Jacobi-Isaacs differential equation rather than the HJB \cite{isaacs1999differential}. 

{\small
\section*{Acknowledgements}
M. Lutter was an intern at Nvidia during this project. 
A. Garg was partially supported by CIFAR AI Chair. 
We also want to thank Fabio Muratore, Joe Watson and the ICML reviewers for their feedback. Furthermore, we want to thank the open-source projects SimuRLacra \cite{simurlacra}, MushroomRL \cite{deramo2020mushroomrl}, NumPy \cite{numpy} and PyTorch \cite{pytorch}.
}

\bibliography{refs}
\bibliographystyle{icml2021}

\clearpage
\section{Appendix}

\begin{figure*}[h]
    \centering
    \includegraphics[width=\textwidth]{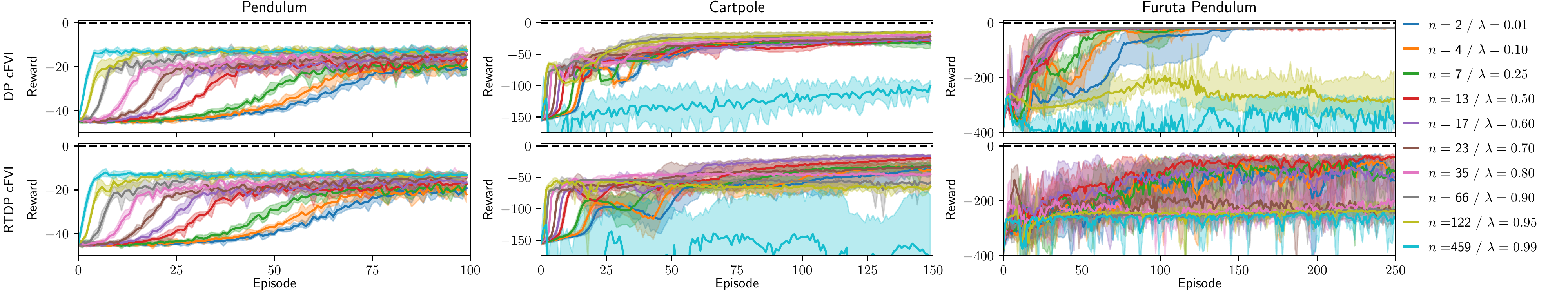}
    \vspace{-2.5em}
    \caption{The learning curves averaged over $5$ seeds for the $n$-step value function target. 
    %for different exponential discounting constants $\lambda = \exp(-\beta \Delta t)$. 
    The shaded area displays the \emph{min/max} range between seeds. The step count is selected such that~$\lambda^{n} \coloneqq 10^{-4}$. Increasing the horizon of the value function target increases the convergence rate to the optimal value function. For very long horizons the learning diverges as it over fits to the current value function approximation. Furthermore, the performance of the optimal policy also increases with roll out length.
    }
    \label{fig:ablation_lambda_2rows}
\end{figure*} 

\section*{Proof of Theorem \ref{theorem:opt_policy}}

\begin{theorem*}
If the dynamics are control affine (\Eqref{eq:affine_dyn}), the reward is separable w.r.t. to state and action (\Eqref{eq:seperable_rwd}) and the action cost $g_c$ is positive definite and strictly convex, the continuous time optimal policy $\pi^{k}$ w.r.t. $V^{k}$ is described by
\begin{gather}
    \pi^{k}(\vx)  = \nabla \tilde{g}_c \left( \mB(\vx)^{T} \nabla_xV^{k} \right) \label{eq:optimal_policy}
\end{gather}
with the convex conjugate $\tilde{g}$ of $g$ and the Jacobian of $V$ w.r.t. the system state $\nabla_x V$.
\end{theorem*} 

\begin{proof}
This proof follows the derivation of Lutter et. al. \yrcite{lutter2019hjb}. This prior work derived the optimal policy $\pi^{*}$ using the Hamilton Jacobi Bellman differential equation (HJB) and generalized the special case described by Doya \yrcite{doya2000reinforcement}. The value iteration update (\Eqref{eq:fvi_update}) is defined as 
\begin{align*}
    V_{\text{tar}}(\vx_t) = \max_{\vu} \: r(\vx_{t}, \vu) + \gamma V(f(\vx_{t}, \vu); \: \psi_k).
\end{align*}
Substituting the assumptions and using the Taylor expansion of the Value function, i.e., 
\begin{align*}
    V(f(\vx_t, \vu)) = V(\vx_t) + \nabla_{x}V^{T} f_c(\vx_t, \vu) \: \Delta t + \mathcal{O}(\Delta t, \vx_t, \vu) \: \Delta t,
\end{align*}
this update can be rewritten - omitting all functional dependencies for brevity - as
\begin{align*}
V_{\text{tar}}  &= \max_{\vu} \:\: r + \gamma V + \gamma \nabla_{x}V^{T} f_c  \Delta t + \gamma \mathcal{O} \Delta t\\
                &= \max_{\vu} \:\: \left[ \gamma \nabla_{x} V^{T} \left(\va + \mB \vu \right) + \gamma \: \mathcal{O} - g_c \right] \: \Delta t + \gamma V + q_c \: \Delta t
\end{align*}
with the higher order Terms $\mathcal{O}(\Delta t, \vx_t, \vu)$. Therefore, the optimal action is defined as 
\begin{align*}
    \vu^{*}_{t} = \argmax_{\vu} \gamma \nabla_{x}V^{T} \left(\va + \mB \vu \right) + \gamma \: \mathcal{O}(\Delta t, \vx_t, \vu) - g_c(\vu).
\end{align*}
In the continuous time limit, the higher order terms~$\mathcal{O}(\Delta t, \vx_t, \vu)$ disappear as these depend on $\Delta t$, i.e., i.e., $\lim_{\Delta t \rightarrow 0} \mathcal{O}(\Delta t, \vx_t, \vu) = 0$. The action is also independent of the discounting as $\lim_{\Delta t \rightarrow 0} \gamma = 1$. Therefore, the continuous time optimal action is defined as 
\begin{align*} 
    \vu^{*}_{t} = \argmax_{\vu} \:  \nabla_{x}V^{T} \mB(\vx) \vu_t - g_{c}(\vu).
\end{align*}
This optimization can be solved analytically as $g_c$ is strictly convex and hence $\nabla g_{c}(\vu) = \vw$ is invertible, i.e., $\vu = \left[ \nabla g_{c}\right]^{-1}(\vw) \coloneqq \nabla \tilde{g}_{c}(\vw)$ with the convex conjugate $\tilde{g}$. The optimal action is described by
\begin{gather*}
\mB(\vx)^T \nabla_{x}V^{*} - \nabla g_c(\vu) \coloneqq 0 \hspace{5pt}
\Rightarrow \hspace{5pt} \vu^{*} = \nabla \tilde{g}_c \left( \mB(\vx)^{T} \nabla_{x}V^{k} \right).
\end{gather*}
Therefore, the value function update can be rewritten by substituting the optimal action, i.e., 
\begin{gather*}
    V_{\text{tar}}(\vx_t) = r\left(\vx_{t}, \nabla \tilde{g}\left(\mB(\vx_t)^{T} \nabla_{x}V^{k}\right)\right) + \gamma V^{k}(\vx_{t+1}; \: \psi_k) \\
    \text{with} \hspace{15pt} \vx_{t+1} = f\left(\vx_{t}, \nabla \tilde{g}(\mB(\vx_t)^{T} \nabla_{x}V^{k}) \right).
\end{gather*}
\end{proof}

\begin{figure*}[h]
    \centering
    \includegraphics[width=\textwidth]{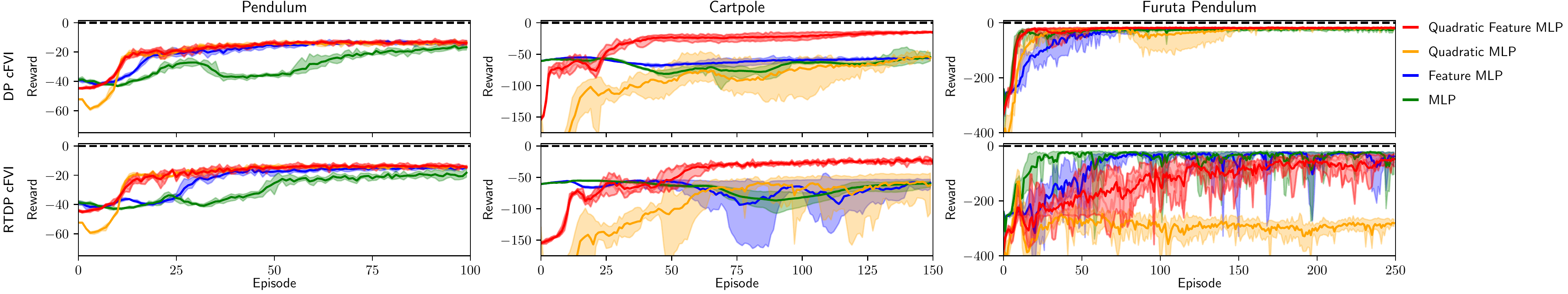}
    \vspace{-2.5em}
    \caption{The learning curves averaged over $5$ seeds for the different model architectures. The shaded area displays the \emph{min/max} range between seeds. All network architectures are capable of learning the value function and policy for most of the tasks. The locally quadratic network architecture increases learning speed compared to the baselines. The structured architecture acts as an inductive bias that shapes the exploration. The global maximum of the locally quadratic value function is guaranteed at $\vx_{\text{des}}$ and hence the initial policy performs hill-climbing towards this point.
    }
    \label{fig:ablation_architecture}
\end{figure*} 

\begin{figure*}[h]
    \centering
    \includegraphics[width=\textwidth]{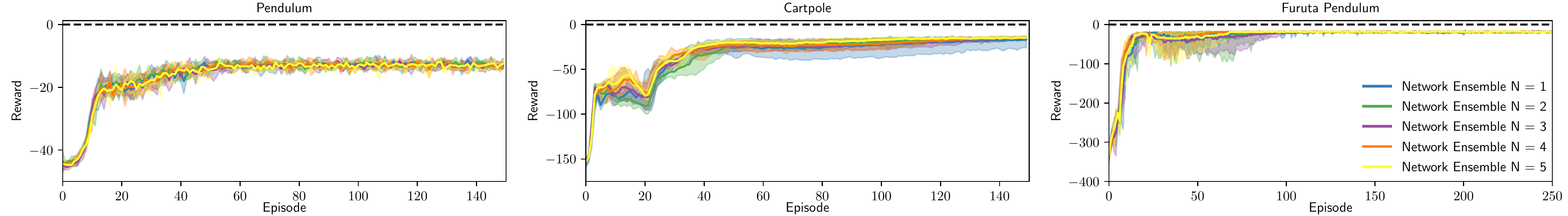}
    \vspace{-2.5em}
    \caption{The learning curves averaged over $5$ seeds with different model ensemble sizes $N$. The shaded area displays the \emph{min/max} range between seeds. The performance of the optimal policy is not significantly affected by the model ensemble. For the cartpole and especially the Furuta pendulum, the larger model ensembles stabilize the training and achieve faster learning and exhibit less variations between seeds.}
    \label{fig:ablation_ensemble}
\end{figure*}

\section*{Experimental Setup}
\textbf{Systems} The performance of the algorithms is evaluated using the \emph{swing-up} the torque-limited pendulum, cartpole and Furuta pendulum. The physical cartpole (Figure \ref{fig:appendix_cartpole}) and Furuta pendulum (Figure \ref{fig:appendix_furuta}) are manufactured by Quanser~\yrcite{quanser}. For simulation, we use the equations of motion and physical parameters of the supplier. Both systems have very different characteristics. The Furuta pendulum consists of a small and light pendulum ($24$g, $12.9$cm) with a strong direct-drive motor. Even minor differences in the action cause large changes in acceleration due to the large amplification of the mass-matrix inverse. Therefore, the main source of uncertainty for this system is the uncertainty of the model parameters. The cartpole has a longer and heavier pendulum ($127$g, $33.6$cm). The cart is actuated by a geared cogwheel drive. Due to the larger masses the cartpole is not so sensitive to the model parameters. The main source of uncertainty for this system is the friction and the backlash of the linear actuator. The systems are simulated and observed with $500$Hz. The control frequency varies between algorithm and is treated as hyperparameter.

\textbf{Baselines} The control performance is compared to Deep Deterministic Policy Gradients (DDPG) \cite{lillicrap2015continuous}, Soft Actor Critic (SAC) \cite{haarnoja2018soft} and Proximal Policy Optimization (PPO) \cite{schulman2017proximal}. The baselines are augmented with uniform domain randomization (UDR) \cite{muratore2018domain}. For the experiments the open-source implementations of MushroomRL \cite{deramo2020mushroomrl} are used. We compare two different initial state distributions ($\mu$). First, the initial pendulum angle $\theta_{0}$ is sampled uniformly, i.e. $\theta_{0} \sim \mathcal{U}(-\pi, +\pi)$. Second, the initial angle is sampled from a Gaussian distribution with the pendulum facing downwards, i.e., $\theta_{0} \sim \mathcal{N}(\pm \pi, \sigma)$. The uniform sampling avoids the exploration problem and generates a larger state distribution of the optimal policy.

\textbf{Reward Function} The desired state for all tasks is the upward pointing pendulum at $\vx_{\text{des}} = \mathbf{0}$. The state reward is described by $q_c(\vx) = -(\vz - \vz_{\text{des}})^T \mQ (\vz - \vz_{\text{des}})$ with the positive definite matrix~$\mQ$ and the transformed state $\vz$. For continuous joints the joint state is transformed to $z_i = \pi^{2} \sin(x_i)$. The action cost is described by $g_{c}(\vu) = - 2 \: \bm{\beta} \: \vu_{\text{max}} / \pi \log \cos(\pi \: \vu / (2 \: \vu_{\text{max}}))$ with the actuation limit~$\vu_{\text{max}}$ and the positive constant $\bm{\beta}$. This barrier shaped cost bounds the optimal actions. The corresponding policy is shaped by~$\nabla\tilde{g}(\vw) = 2\: \vu_{\text{max}} / \pi \: \tan^{-1}(\vw / \bm{\beta})$. For the experiments, the reward parameters are
\begin{align*}
    \text{Pendulum:}& &\mQ_{\text{diag}} &= \left[\textcolor{white}{0}1.0, \: 0.1 \right], & \beta &= 0.5\\
    \text{Cartpole:}& &\mQ_{\text{diag}} &= \left[25.0, \:1.0,\: 0.5, \: 0.1 \right], & \beta &= 0.1\\
    \text{Furuta Pendulum:}& & \mQ_{\text{diag}} &= \left[\textcolor{white}{0}1.0, \: 5.0, \: 0.1, \: 0.1 \right], & \beta &= 0.1\\
\end{align*}

\textbf{Evaluation} 
The rewards are evaluated using $100$ roll outs in simulation and $15$ roll outs on the physical system. If not noted otherwise, each roll out lasts $15s$ and starts with the pendulum downward . This duration is much longer than the required time to swing up. The pendulum is considered balancing, if the pendulum angle is below $\pm 5^{\circ}$ degree for every sample of the last second.

\section*{Extended Experimental Results}

\subsection*{Ablation Study - $N$-step Value Targets}
The learning curves for the ablation study highlighting the importance are shown in Figure \ref{fig:ablation_lambda_2rows}. This figure contains in contrast to Figure \ref{fig:ablation_lambda} also RTDP cFVI. When increasing~$\lambda$, which implicitly increases the $n$-step horizon (Section \ref{sec:cFVI}), the convergence rate to the optimal value function increases. This increased learning speed is expected as \Eqref{eq:contraction} shows that the convergence rate depends on $\gamma^{n}$ with $\gamma < 1$. While the learning speed measured in iterations increases with $\lambda$, the computational complexity also increases. Longer horizons require to simulate $n$ sequential steps increasing the required wall-clock time. Therefore, the computation time increases exponentially with increasing $\lambda$. For example the forward roll out in every iteration of the pendulum increases exponentially from $0.4$s ($\lambda = 0.01$) % , $1.7$s ($\lambda = 0.5$), $15.0$s ($\lambda = 0.95$) 
to $56.4$s ($\lambda = 0.99$). For the Furuta pendulum and the cartpole extremely long horizons of $100+$ steps start to diverge as the value function target over fits to the untrained value function. For RTDP cFVI, the horizons must smaller, i.e., $10$ - $20$ steps. For longer horizons the predicted rollout overfits to the value function outside of the current state distribution, which prevents learning or leads to pre-mature convergence. This is very surprising as even for the true model long time-horizons can be counterproductive due to the local approximation of the value function. Therefore, DP cFVI works bests with $\lambda \in [0.85, 0.95]$ and RTDP cFVI with $\lambda \in [0.45 - 0.55]$.

\subsection*{Ablation Study - Model Architecture}
To evaluate the impact of the locally quadratic architecture described by
\begin{align*}
V(\vx; \: \psi) &= -\left(\vx -  \vx_{\text{des}}\right)^T \mL(\vx;\:\psi) \mL(\vx;\:\psi)^T \left(\vx -  \vx_{\text{des}}\right), % \\
\end{align*}
where $\mL$ is a lower triangular matrix with positive diagonal, we compare this architecture to a standard multi-layer perceptron with and without feature transformation. The learning curves for the ablation study highlighting the importance of the network architecture are shown in Figure \ref{fig:ablation_architecture}. The reward curves are averaged over $5$ seeds and visualize the maximum range between seeds. For most systems all network architecture are able to learn a good policy. The locally quadratic value function is on average the best performing architecture. The structure acts as a inductive bias that shapes the exploration and leads to faster learning. The global maximum is guaranteed to be  at $\vx_{\text{des}}$ as $\mL(\vx;\:\psi) \mL(\vx;\:\psi)^T$ is positive definite. Therefore, the initial policy directly performs hill-climbing towards the balancing position.  Only for the cartpole the other network architectures fail. For this system, these architectures learn a local optimal solution of balancing the pendulum downwards. This local optima the conservative solution as the cost associated with the cart position is comparatively high to avoid the cart limits on the physical system. Therefore, stabilizing the pendulum downwards is better compared to swinging the pendulum up and failing at the balancing. The locally quadratic network with the feature transform, learns the optimal policy for the cartpole. This architecture avoids the local solution as the network structure guides the exploration to be optimistic and the feature transform simplifies the value function learning to learn a successful balancing. 

\subsection*{Ablation Study - Model Ensemble}
The learning curves for different model ensemble sizes are shown in Figure \ref{fig:ablation_ensemble}. The model ensemble does not significantly affect the performance of the final policy but reduces the variance in learning speed between seeds. The variance between seeds also increases. The reduced variance for the model ensembles is caused by the smoothing of the network initialization. The mean across different initial weights lets the initialization be more conservative compared to a single network. For the comparatively small value function networks (i.e., 2-3 layers deep and 64 - 128 units wide), we prefer the network ensembles as the computation time does not increase when increasing the ensemble size. If the individual networks are batched and evaluated on the GPU the computation time does not increase. The network ensembles could also be evaluated at $500$Hz for the real-time control experiments using an Intel i7 9900k.

\end{document}